\pdfoutput=1

\documentclass[11pt]{article}

\usepackage[preprint]{acl}
\usepackage{booktabs}
\usepackage{amsfonts}
\usepackage{amsmath}
\usepackage{multirow}
\usepackage{amsthm}
\usepackage{algorithm}
\usepackage{algorithmic}
\newtheorem{theorem}{Theorem}[section]
\newtheorem{assumption}{Assumption}[section]
\newtheorem{definition}{Definition}[section]
\newtheorem{proposition}{Proposition}[section]
\newtheorem{corollary}{Corollary}[theorem]
\newtheorem{lemma}[theorem]{Lemma}
\newtheorem*{remark}{Remark}
\usepackage{times}
\usepackage{latexsym}

\usepackage[T1]{fontenc}

\usepackage[utf8]{inputenc}

\usepackage{microtype}

\usepackage{inconsolata}

\usepackage{graphicx}

\title{A statistically consistent measure of semantic uncertainty using Language Models}

\author{Yi Liu \\
  Prorata.ai\\
  Bellevue, Washington, USA \\
  \texttt{yi@prorata.ai} 
}

\usepackage{xr}
\begin{document}
\maketitle
\begin{abstract}
To address the challenge of quantifying uncertainty in the outputs generated by language models, we propose a novel measure of semantic uncertainty, semantic spectral entropy, that is statistically consistent under mild assumptions. This measure is implemented through a straightforward algorithm that relies solely on standard, pretrained language models, without requiring access to the internal generation process. Our approach imposes minimal constraints on the choice of language models, making it broadly applicable across different architectures and settings. Through comprehensive simulation studies, we demonstrate that the proposed method yields an accurate and robust estimate of semantic uncertainty, even in the presence of the inherent randomness characteristic of generative language model outputs.
\end{abstract}

\section{Introduction}

\label{introduction}
{\color{white}..} The birth of large language models (LLMs) has given rise to the possibility of a wide range of industry applications \cite{touvron2023llama,chowdhery2023palm}. One of the key applications of generative models that has garnered significant interest is the development of specialized chatbots with domain-specific expertise such as legal and healthcare. These applications illustrate how generative models can improve decision-making and improve the efficiency of professional services in specialized fields.

A significant challenge hindering the widespread industrial deployment of generative models is the lack of consistency in LLM outputs across multiple runs given the same input \cite{amodei2016concrete, hendrycks2021unsolved}. Without delving into the intricate mathematical architecture of generative language models, one can observe that these models function as stochastic generators: they take an input string and produce an output string. Prior studies \cite{wang2020contextual, wang2023cost, song2024good} have demonstrated that the output for a fixed input can vary across runs, with the degree of variation influenced by factors such as temperature, top-k sampling, top-p sampling, and repetition penalties.

In a recent study, Atil et al. (2024) \cite{atil2024llm} conducted experiments using six deterministically configured LLMs—each with temperature set to 0 and top-p set to 1—across eight common tasks, with five identical trials per task. The goal was to assess output repeatability under conditions designed to minimize randomness. Surprisingly, none of the models exhibited complete consistency in their outputs across all tasks. For more complex tasks, such as those involving college-level mathematics, the models frequently produced lexically distinct outputs across trials, resulting in zero consistency with respect to exact string matching. However, the authors also observed that many of these outputs were semantically equivalent despite their lexical differences, suggesting that syntactic variation does not necessarily imply divergence in meaning.

Variations in observed responses have been attributed to the use of GPUs during large language model (LLM) inference, where premature rounding and non-deterministic computation can introduce discrepancies \cite{nvidia2024, atil2024llm}. Consequently, it is reasonable to regard LLM outputs as inherently stochastic, with the elimination of this randomness being practically challenging—if not infeasible. This poses a significant concern for the commercial deployment of LLMs: when errors occur (i.e., when an output deviates from the expected result given a specific input), it becomes difficult to disentangle implementation-induced randomness from genuine model deficiencies. This ambiguity complicates efforts to make targeted improvements to the system. In response to this challenge, we propose the need for an empirical measure of uncertainty that evaluates the semantic variability of LLM outputs, rather than relying on surface-level lexical comparisons. Such a measure can be interpreted as a variance estimator that captures the semantic dispersion within a set of generated outputs.

Most prior studies on uncertainty in foundation models for natural language processing (NLP) have focused primarily on the calibration of classifiers and text regressors \cite{jiang2021can, desai2020calibration, glushkova2021uncertainty}. Other research has addressed uncertainty by prompting models to evaluate their own outputs or fine-tuning generative models to predict their own uncertainty \cite{linteaching, kadavath2022language}. However, these approaches require additional training and supervision, making them difficult to reproduce, costly to implement, and sensitive to distributional shifts. 

 Our work follows from a line of work inline with the concept of semantic entropy proposed in \cite{kuhn2023semantic, nikitin2024kernel,duan-etal-2024-shifting,lin2023generating}. \cite{kuhn2023semantic} explore the entropy of the generated text by assigning semantic equivalence to the pairs of text and subsequently estimating the entropy. Similarly, \cite{nikitin2024kernel} and \cite{lin2023generating} utilize graphical spectral analysis to enhance empirical results. However, a notable limitation in the entropy estimators proposed by \cite{kuhn2023semantic} and \cite{nikitin2024kernel} is their reliance on token likelihoods when assessing semantic equivalence, which may not always be accessible. Furthermore, \cite{kuhn2023semantic} acknowledge that the clustering process employed in their framework is susceptible to the order of comparisons, introducing variation into the results. 

Moreover, prior work in this area has primarily focused on the empirical performance of entropy estimators. While such methods have shown promising results in practice, to the best of our knowledge, no existing studies have provided a formal theoretical analysis demonstrating that these estimators converge to the true entropy value as the sample size increases under an assumed underlying model. Investigating the theoretical properties of our proposed measure enables a deeper understanding of how factors such as the number of clusters and the sample size influence the estimator’s behavior and consistency.

Our approach aims to address these limitations by developing a robust theoretical framework for the clustering procedure, ensuring convergence properties and reducing the variability observed in prior methodologies. We introduce a theoretically analyzable metric, semantic spectral entropy, to quantify the semantic variation within a collection of texts. This uncertainty quantification measure is motivated by the observation that many generated strings, though lexically and syntactically distinct, may express equivalent semantic content. To identify such semantic equivalences, we leverage standard generative language models (LMs) as tools for semantic evaluation.

Importantly, we recognize that the language model used to assess semantic similarity is itself imperfect, introducing potential noise due to its inability to provide uniformly accurate judgments across all text pairs. To mitigate this issue, we employ spectral clustering, a well-established technique known to be statistically consistent under mild assumptions on the underlying data-generating process. This provides robustness and theoretical guarantees for our proposed measure. Specifically, we show that semantic spectral entropy is statistically consistent under a weak assumption on the language model. To the best of our knowledge, this constitutes the first semantic uncertainty measure with formally proven convergence properties.

\section{Semantic spectral entropy}
\label{methodology}
\subsection{Semantic entropy}

{\color{white}..}We begin with a collection of $n$ textual items, denoted $\mathcal{T} = (t_1, \ldots, t_n)$. In contrast to the approach in \cite{kuhn2023semantic}, we assume access only to the texts in $\mathcal{T}$, without requiring knowledge of the underlying generative process. While the methodology is designed with the goal of assessing uncertainty in outputs generated by large language models (LLMs), it is important to note that it does not necessitate $\mathcal{T}$ to be the result of a generative model. That is, the proposed method operates independently of the generation mechanism and can be applied in broader contexts. To evaluate the semantic uncertainty of the texts within a specific application domain, we introduce a theoretically grounded measure, which we term semantic spectral entropy.

While the proposed measure can be interpreted as a variance estimator for the semantics of $\mathcal{T}$, a central reason for not adopting the conventional notion of variance as a measure of uncertainty lies in the difficulty of defining a meaningful "mean" for semantic probability distributions. Unlike numerical data, semantic representations lack a natural central point around which variation can be measured. Although one might define an arbitrary reference point—such as a canonical or "standard" answer in the context of a chatbot—computing variance relative to such a reference introduces inherent bias and undermines the objectivity of the measure. Our approach, therefore, avoids reliance on predefined anchors and instead captures semantic dispersion directly, without the need for a fixed semantic center.

In contrast, entropy is a well-established measure of variation, particularly for multinomial distributions. For a distribution \( \mathcal{P}(t) \) over a set of semantic clusters \( \{C_1, \cdots, C_k\} \), the entropy \( \mathcal{E} \) is defined as:
\begin{equation}
\label{equ:entropy}
\mathcal{E}(t) = - \sum_{i} p(t \in C_i)\log p(t \in C_i).
\end{equation}
This formulation captures the uncertainty or disorder associated with assigning a given text \( t \) to one of the clusters. Consequently, it provides a quantitative measure of semantic uncertainty that avoids the biases introduced by arbitrary reference points.

To estimate the entropy for a given data set \( t_1, \cdots, t_n \), we first calculate the number of occurrences of each text \( t_i \) in each group \( C_j \). This is achieved by computing:
\[
n_j = \sum_{i=1}^n \mathbb{I}(t_i \in C_j),
\]
where \( \mathbb{I}(t_i \in C_j) \) is an indicator function that equals 1 if \( t_i \) belongs to the cluster \( C_j \), and 0 otherwise. 

Next, the true probability \( p(t \in C_j) \) is approximated using the empirical distribution:
\[
\bar{p}(t \in C_j) = \frac{n_j}{n},
\]
which represents the fraction of texts assigned to cluster \( C_j \). Using this empirical distribution, the empirical entropy is defined as:
\[
\bar{\mathcal{E}}(\mathcal{T}) = - \sum_{j} \bar{p}(t \in C_j) \log \bar{p}(t \in C_j).
\]
This measure provides a practical estimation of semantic entropy based on observed data.

One critical step in this process is clustering the texts $t_i$ into disjoint groups. To do so, it is sufficient to define a relationship between $t_i \sim t_j$, such that they satisfy the properties of equivalence relation. Specifically, one needs to demonstrate 
\begin{enumerate}
    \item Reflexivity: For every $t_i$, we have $t_i \sim t_i$, meaning that any text is equivalent to itself.
    \item Symmetry: If $t_i \sim t_j$, then $t_j \sim t_i$, meaning that equivalence is bidirectional.
    \item Transitivity: If $t_i \sim t_j$ and $t_j \sim t_k$, then $t_i \sim t_k$, which means that equivalence is transitive.
\end{enumerate}

It turns out the existence of an equivalence equation is both a necessary and sufficient condition for a definition of a breakdown of $\mathcal{T}$ into disjoint clusters \cite{liebeck2018concise}. In light of this, defining $\sim$ should be based on the linguist properties of entropy measurement. 

 Direct string comparison, defined as \( t_i \sim t_j \) if and only if \( t_i \) and \( t_j \) share identical characters, reflects lexicon equality and constitutes an equivalence relation. However, this criterion is overly restrictive. In a question-and-response context, a more appropriate equivalence relation might be defined as \( t_i \sim t_j \) if and only if \( t_i \) and \( t_j \) yield identical scores when evaluated by a language model (LM) prompt. This criterion, however, requires an answer statement as a point of reference. We are more interested in a stand-alone metric that can capture the semantic equivalence. For example, consider the sentences \( t_1 = \text{"Water is vital to human survival"} \) and \( t_2 = \text{"Humans must have water to survive"}\). Despite differences in language, both sentences convey the same underlying meaning.

To address such challenges, \cite{kuhn2023semantic,nikitin2024kernel} propose an equivalence relation wherein \( t_i \sim t_j \) if and only if \( t_i \) is true if and only if \( t_j \) is true. This formulation ensures that two texts, \( t_i \) and \( t_j \), belong to the same equivalence class if they are logically equivalent. This broader definition allows for greater flexibility and applicability in assessing semantic equivalence beyond superficial lexical similarity.\cite{copi2016introduction}. We will present their argument as a proposition where we will put the verification in the appendix
\begin{proposition}
    \label{prop:equ}
    The relation $t_i \sim t_j$ if "$t_i$ is true if and only if $t_j$ is true" is an equivalence relation.
\end{proposition}



    
In light of the fact that equivalence relations can be defined arbitrarily based on the needs of the user. We propose that the determination of equivalence relations, denoted as $\sim$, is performed through a LM that generates responses independently of the specific generation of terms $t_1, \dots, t_n$. However, we do not assume that we have access to probability distribution of the tokens as proposed by \cite{kuhn2023semantic,nikitin2024kernel} which is not always available. Rather, we just require a LM which can identify this relationship. In fact, this LM can be off-the-shelf language model with a crafted prompt, something we will use in our simulation studies. The error from LM will be removed in the spectral clustering algorithm at the later stage. By leveraging this LM, we define a function $e:{\mathcal{T}, \mathcal{T}}\rightarrow {0,1}$, which is formally expressed as follows: \begin{equation} e(t_i, t_j) = \begin{cases} 1 & \text{if } t_i \sim t_j, \\ 0& \text{otherwise.} \end{cases} \end{equation}

However, since the function relies on an LM, $e(t_i, t_j)$ can be viewed as a Bernoulli random variable (given that we have established that there are inherent uncertainty in the LMs), whose value is dependent on the terms $t_i$ and $t_j$.\cite{kuhn2023semantic} did not address this issue but instead offers adopting a very powerful entailment identification model which the authors trust to identify the equivalence relation perfectly. In contrast, we suggest modeling the outputs of the LM as a random graph with an underlying distribution. In this framework, $t_i$ and $t_j$ represent nodes, while $e(t_i, t_j)$ are random variables that indicate the presence of an edge between the two nodes. Specifically, when $t_i \sim t_j$, the edge existence is governed by the following probability distribution: \begin{equation} \label{eqn:equation_p} e(t_i, t_j) = \begin{cases} 1 & \text{with probability } p, \\ 0 & \text{with probability } 1-p. \end{cases} \end{equation} Conversely, when $t_i \not\sim t_j$, the edge existence follows a different probability distribution: \begin{equation} \label{eqn:equation_q} e(t_i, t_j) = \begin{cases} 1 & \text{with probability } q, \\ 0 & \text{with probability } 1-q. \end{cases} \end{equation}

To mitigate the inherent randomness introduced by the LLM, we propose leveraging spectral clustering to identify clusters of semantically similar texts.
\subsection{Spectral clustering}

{\color{white}..} To compute semantic entropy, it is crucial to identify the clusters of nodes and count the number of nodes within each cluster. Identifying these clusters in a random graph is analogous to detecting clusters in a stochastic block model \cite{holland1983stochastic}. We propose employing the spectral clustering algorithm, with the number of clusters $K$ specified in advance, as an effective approach for this task.

Spectral Clustering is a well-established algorithm for graph clustering, supported by strong theoretical foundations and efficient implementations \cite{shi2000normalized, lei2015consistency, su2019strong, scikit-learn}.  To compute semantic entropy, we aim to cluster a random graph with adjacency matrix $E$ where $E_{ij} = e(t_i, t_j)$, representing the pairwise similarity between text elements $t_i$ and $t_j$. 

We begin by computing the Laplacian matrix $L =  D-E$ where $D$ is the degree matrix.  This is followed by the decomposition of the eigenvalue of $L$. Next, we construct the matrix formed by the first $K$ eigenvectors of $L$ denoted $\hat{U} \in \mathbb{R}^{n\times K}$. This matrix serves as input to an appropriate $(1+\epsilon)-$ k-means clustering algorithm \cite{kumar2004simple,choo2020k}.

The output of this procedure is $K$ distinct clusters $C_1,\cdots C_K$. For each text element $t_i$, we assign a corresponding vector $g_{i}$ where 
$$ g_{ij} = \begin{cases} 1 \text{ if } t_i \in C_j\\
    0 \text{ otherwise }
\end{cases}$$
This binary indicator vector $g_i$ encodes the cluster membership for each text element $t_i$

Finally, we compute the estimated entropy based on the number of texts within each cluster. The entropy $\hat{\mathcal{E}}$ can be approximated using the following formula:
\begin{equation}
\hat{\mathcal{E}}(\mathcal{T}) = - \sum_{j=1}^k\hat{p}( C_j) \log(\hat{p}( C_j)),
\end{equation}
where $\hat{p}(C_j) = \frac{1}{n}\sum_{i=1}^n g_{ij}$.This expression represents the empirical entropy based on the distribution of texts among the $K$ clusters, providing a measure of the uncertainty or diversity within the semantic structure of the data.
\subsection{Full algorithm and implementation}
{\color{white}..} We merge the process of finding sermantic entropy with spectral clustering to present the full algorithm as Algorithm \ref{algo:1}: Sermantic Spectral Entropy. 
\begin{algorithm}
\begin{algorithmic}
    \STATE Begin with $\mathcal{T} = \{t_1, \cdots t_n\}$
    \FOR{$i, j \in \{1,\cdots n\} \times \{1, \cdots n\}, i\neq j$}
    \STATE Use LLM to compute $E_{i,j} = e(t_i, t_j)$. 
    \ENDFOR
    \STATE Find the Laplacian of $E$, $L = D -E$
    \STATE Compute the first $K$ eigenvectors $u_1,\dots,u_k$ of $L$ and the top $K$ eigenvalues $\lambda_1,\cdots \lambda_k$.
    \STATE Let $\hat{U} \in \mathbb{R}^{n\times k}$ be the matrix containing the vectors $u_1,\dots,u_k$ as columns.
    \STATE Use $(1+\epsilon)$ K-means clustering algorithm to cluster the rows of $U$
    \STATE Let $g_{ij}$ be an $(1+\epsilon)-$approximate solution to a $K-$means clustering algorithm
    \STATE Compute $\hat{\mathcal{E}}(\mathcal{T})$ using $ g_{ij}$
\end{algorithmic}
\caption{\label{algo:1} Sermantic Spectral Entropy }
\end{algorithm}

This polynomial-time algorithm is characterized by the largest computational cost associated with the determination of $E_{ij}$. However, computing $E_{ij}$ is embarrassingly parallel, meaning that it can be efficiently distributed across multiple processing units. Furthermore, there are well-established implementation, such as Microsoft Azure's Prompt-Flow \cite{esposito2024programming} and LangChain \cite{mavroudis2024langchain} that facilitate the implementation of parallel workflows, making it feasible to deploy such parallelized tasks with relative ease.
\subsection{Finding K}

{\color{white}..} A notable limitation of this analysis is the unavailability of $K$ in the direct computation of semantic spectral entropy. However, the determination of $K$ for stochastic block model has been well studied \cite{lei2016goodness,wang2017likelihood,chen2018network}. We will describe the cross-validation approach \cite{chen2018network} in detail. The principle behind cross-validation involves predicting the probabilities associated with inter-group connections ($p$) and intra-group connections ($q$). If the estimated value of $K$ is too small, it fails to accurately recover the true underlying probabilities; conversely, if $K$ is too large, it leads to overfitting to noisy data. This approach has the potential to recover the true cluster size under relatively mild conditions.

\section{Theoretical Results}
\label{theory}

{\color{white}..} Our theoretical analysis involves a proof that the estimator is strongly consistent, i.e. the estimator converges to true value almost surely, and an analysis of its rate with respect to the number of cluster $K$. 

We divide our analysis into two subsections. The first subsection examines a fixed set of $\mathcal{T} = {t_1, \dots, t_n}$, which is assumed to exhibit some inherent clusters $C_1, \dots, C_K$. Under the assumption of perfect knowledge of these clusters, the empirical entropy $\bar{\mathcal{E}}$ can be determined. The primary focus in this subsection is on the performance of spectral clustering algorithms. The second subsection explores a scenario in which there exists an underlying generative mechanism that allows for the infinite generation of $t_i$. In this case, we permit $K$ to increase with $n$, though at a significantly slower rate. This scenario is particularly relevant for evaluating the performance of RAG in the context of continuous generation of results in response to a given query.

\subsection{Performance of spectral clustering algorithms}
{\color{white}..}  We model the LM determination of $e(t_i,t_j)$ as a random variable, as described in Equations \ref{eqn:equation_p} and \ref{eqn:equation_q}. In the theoretical analysis presented here, we assume that the number of clusters, $K$, is known and fixed. To derive various results, we first establish the relationship between the difference $|\bar{\mathcal{E}}(\mathcal{T}) - \hat{\mathcal{E}}(\mathcal{T})|$ and the miscluster error, denoted $M_\text{error}$.
\begin{lemma}
 \label{lemma:error}
Suppose that there exists $0<c_2<1$ such that $2Kn_{\min}/n \geq c_2$, 
\begin{equation}
     |\hat{\mathcal{E}}(\mathcal{T}) - \bar{\mathcal{E}}(\mathcal{T})|\leq h\left(\frac{2K}{c_2}\right) \left|\frac{1}{n} (M_\text{error})\right| 
\end{equation}
where $h(x) = \left(x+\log\left(x\right)\right)$.
\end{lemma}
The proof is presented in the Appendix section \ref{Appendix:proofoflemma:error}. 
We begin by presenting the result of strong consistency for the spectral clustering algorithm.  

\begin{theorem}
\label{the:strongConsistensy}
Under regularity conditions, the estimated entropy empirical entropy $\hat{\mathcal{E}}(\mathcal{T})$ is strongly consistent with the empirical entropy, i.e. 
\begin{equation}
    |\bar{\mathcal{E}}(\mathcal{T}) - \hat{\mathcal{E}}(\mathcal{T}) | \rightarrow 0 \text{ almost surely }
\end{equation}
\end{theorem}
The proof is provided in the Appendix section \ref{appendix:sec:the:strongconsistency}. This establishes strong consistency result that we aim to present. At the same time, we also want to show the finite sample properties of the estimator $\hat{\mathcal{E}}(\mathcal{T})$.

\begin{theorem}
    \label{the:finite_sample}
    If there exists $0<c_2\leq1$ and $\lambda > 0$ such that $2Kn_{\min}/n \geq c_2$, and $p = \alpha_n = \alpha_n(q + \lambda) $, where $\alpha_n \geq \log(n)$ then with probability at least $1-\frac{1}{n}$

\begin{equation}
|\bar{\mathcal{E}}(\mathcal{T}) - \hat{\mathcal{E}}(\mathcal{T}) |  \leq h\left(\frac{2K}{c_2}\right) \frac{n_{\max }}{4c_2^2n_{\min }^{2} \alpha_{n}K^2} 
\end{equation}
where $h(x) = \left(x+\log\left(x\right)\right)$, $n_{\max} = \max_j\{n_j : j = 1,\dots K\}$, and $n_{\min} = \min_j\{n_j : j = 1,\dots K\}$.
\end{theorem}
The full proof is provided in the appendix section \ref{appendix:proofofthe:finite_sample}. A brief outline of the proof is as follows: we begin by using the results from \cite{lei2015consistency}, which establish the rate of convergence for the stochastic block model. Next, we relate the errors of the spectral clustering algorithm to the errors in the empirical entropy, using the lemma \ref{lemma:error} to establish this connection.

\begin{remark}
    This result is particularly relevant for computing semantic entropy, as the output generated by LMs is produced with a probability that is independent of $n$. As a result, we have $\alpha_n = O(1)$. Assuming balanced community sizes, the convergence rate is therefore $O(\frac{1}{n})$. This is formally stated in the following corollary:
\end{remark}

\begin{corollary} \label{corollary:rate} If there exists a constant $0 < c_2 \leq 1$ such that $2Kn_{\min}/n \geq c_2$ and $\alpha_n = alpha >0$, then there exists a constant $\alpha$ such that with probability at least $1 - \frac{1}{n}$, \begin{equation} |\bar{\mathcal{E}}(\mathcal{T}) - \hat{\mathcal{E}}(\mathcal{T})| \leq h\left(\frac{2K}{c_2}\right) \frac{1}{c_2^4 \alpha n}. \end{equation} \end{corollary}

The proof of this result is provided in the Appendix section \ref{appendix:proofofcorollary:rate}. 
\begin{remark}
    In particular, we observe that the convergence rate is $O\left(\frac{1}{n}\right)$. This means that the error associated with spectral clustering is small, and our estimated entropy converges to the empirically entropy quickly.
\end{remark}

\subsection{Performance under a generative model}

{\color{white}..} In practical terms, we assume the presence of a generator, specifically an RAG, that produces identically distributed independent random variables $t_i$' that collectively form semantic clusters $C_1 \dots C_K$. In essence, we have $t_i \sim G$ such that $t_i \in C_j$ with probability $p(C_j)$. In this model, there is a true value of entropy $\mathcal{E}(\mathcal{T})$ given in Equation \ref{equ:entropy}, and we want to find the convergence rate of our method. 
\begin{theorem}
    \label{the:final} If there exists a constant $\alpha $ such that $p = \alpha  = \alpha(q + \lambda) $, then with probability at least $1-\frac{3}{n}$,
    \begin{equation}
    \label{eqn:final_the}
       \begin{array}{cc}
         |\mathcal{E} - \hat{\mathcal{E}}|&  \leq h\left(\frac{1}{p_{\min}}\right)K\sqrt{\frac{1}{2n}\log\left(2Kn\right)}\\
         & +h\left(\frac{1}{m(n)p_{\min}}\right)\frac{1}{16K^4m(n)^4p_{\min}^4n}
    \end{array} 
    \end{equation}

where $m(n) = \left(1- \sqrt{2\log(nK)/np_{\min}}\right)$ and $p_{\min} = \min\{p(C_1)\dots p(C_K)\}$.
\end{theorem}
Most of the material used for this proof is presented in Corollary \ref{corollary:rate}. 
\begin{proof}
Consider the following equality
$$|\mathcal{E} - \hat{\mathcal{E}}| \leq |\mathcal{E} -\bar{\mathcal{E}} +\bar{\mathcal{E}}-  \hat{\mathcal{E}}| \leq |\mathcal{E} -\bar{\mathcal{E}}| + | \bar{\mathcal{E}}-  \hat{\mathcal{E}}|,$$  
 We know that there are three sufficient conditions for Equation \ref{eqn:final_the}. These are
\begin{enumerate}
    \item[C1:]$|\mathcal{E} -\bar{\mathcal{E}}| \leq  h\left(\frac{1}{p_{\min}}\right)K\sqrt{\frac{1}{2n}\log\left(2Kn\right)},$
    \item[C2:]$ \exists c_2 \text{ such that } 0 < c_2 \leq 1$ and $2Kn_{\min}/n \geq c_2,$ 
    \item[C3:] $ |\bar{\mathcal{E}}(\mathcal{T}) - \hat{\mathcal{E}}(\mathcal{T})| \leq h\left(\frac{2K}{c_2}\right) \frac{1}{c_2^4 n}.$
\end{enumerate}
Then, using union bound
\begin{align*}
    \mathbb{P}(\text{Not (\ref{eqn:final_the})}) &\leq \mathbb{P}( \text{Not C1 or Not C2 or Not C3})\\
    &\leq \mathbb{P}( \text{Not C1}) + \mathbb{P}( \text{Not C2})+ \mathbb{P}( \text{Not C3}).
\end{align*}
In Lemma \ref{lemma:final1} and \ref{lemma:Final2} of the appendix, we show that $|\mathcal{E} -\bar{\mathcal{E}}| \geq  h\left(\frac{1}{p_{\min}}\right)K\sqrt{\frac{1}{2n}\log\left(2Kn\right)}$ with probability at most $\frac{1}{n}$.

In Lemma \ref{Lemma:Final3} of the Appendix, we show that setting $c_2 = 2K\left(1-\sqrt{\frac{2\log(nK)}{np_{\min}}}\right)p_{\min}$, we have $2Kn_{\min}/n< c_2$ with probability at most $\frac{1}{n}$.

Finally, the corollary \ref{corollary:rate} tells us that $ |\bar{\mathcal{E}}(\mathcal{T}) - \hat{\mathcal{E}}(\mathcal{T})| > h\left(\frac{2K}{c_2}\right) \frac{1}{c_2^4 n}$ occurs with probability at most $\frac{1}{n}$.
\end{proof}
\begin{remark}
One observation is that the empirical entropy converges to true entropy at a rate slower than that of estimated entropy to the empirical entropy. This is natural since each $t_i$ has the opportunity to make a $n-1$ connection with other $t_j$s, resulting in $n(n-1)/2$ independent observations, whereas each generator generates only $n$ independent observations.
\end{remark}
\subsection{Discussion on $K$}
{\color{white}..} An intriguing question to consider is the rate at which \( K \), the number of clusters, can grow with \( n \), the number of texts, as it is natural to expect \( K \) to increase with \( n \). Focusing solely on the spectral clustering algorithm, the error is characterized as \( O((K + \log(K))/n) \). Thus, under the condition \( K = o(n^{1-\delta}) \) for some \( \delta > 0 \), we have \( |\bar{\mathcal{E}}(\mathcal{T}) - \hat{\mathcal{E}}(\mathcal{T})| \to 0 \) in probability. In contrast, when considering a scenario involving a generative model, a stricter condition is required. Specifically, \( K \) must satisfy \( K = o(n^{1/2 - \delta}) \), with \( \delta > 0 \), to ensure \( |\mathcal{E}(\mathcal{T}) - \hat{\mathcal{E}}(\mathcal{T})| \to 0 \) in probability.

\section{Simulation and data studies}
\label{simulation}

{\color{white} .. }As this paper focuses more on the theoretical analysis of semantic spectral entropy with respect to variable $n$ and $K$, we decide against using the evaluation method proposed in \cite{kuhn2023semantic,duan-etal-2024-shifting, lin2023generating} in favor of constructing a simulation where we know the true entropy $\bar{\mathcal{E}}$. This allows us to better analyze how $|\bar{\mathcal{E}} -\hat{\mathcal{E}}|$ changes with choice of generator $e$, $K$ and $n_{\min}$.

To construct a non-trivial simulation for this use case, we evaluate the performance of our algorithms within the context of an unordered set of elementary proposition statements that has no logical interconnections. This approach draws upon the philosophical framework defined by \citep{wittgenstein2023tractatus} in Tractatus Logico-Philosophicus, where each elementary proposition represents a singular atomic fact. Within this framework, texts containing an identical set of elementary propositions are deemed semantically equivalent. The primary advantage of this experimental design lies in its efficiency, as it facilitates the generation of thousands of samples with minimal generator propositions, all while maintaining knowledge of the ground truth.

For example, we can consider a list of things that a hypothetical individual "John" likes to do in his free time: 
\begin{itemize}
    \item Running/Jogging 
    \item ...
\end{itemize}

To generate a cluster of text from this set of hobbies, we begin by randomly selecting \( M \) items from a total of \( N \) items in the list to formulate the compound proportion. This selection process yields \( \binom{N}{M} \) potential subset of hobbies and we know that two subsets of hobbies are the same as long as their elements are the same. Next, to create individual text samples \( t_i \) within the group, we randomly permute the order of the \( M \) selected elements in the subset. This permutation process generates \( M! \) unique samples for each combination of hobbies. Finally, the hobbies are placed in their permuted order in a sentence like that in the following. 
\begin{quote}
"In his free time, John likes hobby $1$, hobby $2$, hobby $3$, ..., and hobby $M$ as his hobbies."
\end{quote}
In order to prevent models from relying on sentence structure, some of these sentences are being designed. 

We replicate this simulation setup in two different settings. In the first setting (Experiment 1), we consider ten common hobbies that a hypothetical individual engages in during their spare time. In the second setting (Experiment 2), we examine ten historical events that occurred on December 3, which we collect from Wikipedia \cite{wiki}. These two settings allow us to analyze the impact of text length by comparing long and short text scenarios.

We utilize Microsoft Phi-3.5 \cite{abdin2024phi}, OpenAI GPT3.5-turbo (denoted as GPT) \cite{hurst2024gpt}, A21-Jamba 1.5 Mini (denoted as A21) \cite{lieber2021jurassic}, Cohere-command-r-08-2024 (denoted as Cohere) \cite{Ustun2024AyaMA}, Ministral-3B (denoted as Minstral) \cite{jiang2023mistral} and the Llama 3.2 70B model (denoted as Llama) \cite{dubey2024llama} as \( e \). These models are lightweight, off-the-shelf language models that are cost-effective to deploy and exhibit efficiency in generating outputs, thereby offsetting the computational cost of determining sermantic relationships. The exact prompt used to generate the verdict is specified in Appendix \ref{appendix_sec:prompt_engineering}. \footnote{Code are in https://github.com/yiliu9090/sermantic-spectral-entropy}

\begin{table*}[ht]
\centering
\begin{tabular}{|l|rrr|rrr|rrr|rrr|}
\toprule
\multicolumn{13}{|l|}{Experiment 1 (Hobbies)} \\
\midrule
ratio & \multicolumn{3}{|r|}{0.2,0.3,0.5} & \multicolumn{3}{r|}{0.3,0.3,0.4} & \multicolumn{3}{r|}{0.5,0.5}&$p-q$& $p$& q\\
datasize & 30 & 50 & 70 & 30 & 50 & 70 & 30 & 50 & 70 & & &\\
\midrule
LLAMA & 0.36 & 0.49 & 0.44 & 0.34 & 0.43 & 0.46 & 0.30 & 0.27 & 0.26& 0.17 &0.17 &0.00 \\
MINISTRAL & 0.22 & 0.27 & 0.13 & 0.25 & 0.23 & 0.21 & 0.14 & 0.22 & 0.21  &0.22 & 0.99 & 0.77 \\
COHERE & 0.04 & 0.02 & 0.06 & 0.02 & 0.03 & 0.00 & 0.00 & 0.00 & 0.00 & 0.55 & 0.61 & 0.05 \\
A21 & {\bf 0.05} & {\bf 0.00} & {\bf 0.00} & {\bf 0.00} & 0.01 & {\bf 0.00} & {\bf 0.00} & {\bf 0.00} & {\bf 0.00} & 0.81 & 0.96 & 0.15\\
PHI & 0.08 & 0.07 & 0.07 & 0.03 & 0.03 & 0.00 & 0.00 & 0.00 & 0.00 & 0.67 & 0.67 & 0.01\\
GPT & 0.06 & 0.02 & 0.00 & 0.01 & {\bf 0.00} &  0.00 & 0.00 & 0.00 & 0.00 & 0.80 & 0.87 & 0.07\\

\toprule
 \multicolumn{13}{|l|}{Experiment 2 (Historical Events)} \\
\midrule
LLAMA &  0.04 & 0.02 & 0.05 & 0.03 & 0.03 & 0.07 & 0.04 & 0.00 & 0.00& 0.75 & 0.75 & 0.00\\
MINISTRAL & 0.20 & 0.20 & 0.19 & 0.17 & 0.24 & 0.19 & 0.04 & 0.07 & 0.19 & 0.25 & 1.00 & 0.75\\
COHERE & 0.08 & 0.04 & 0.09 & 0.08 & 0.09 & 0.04 & 0.00 & 0.01 & 0.04 & 0.52 & 0.98 & 0.46 \\
A21 & 0.09 & 0.14 & 0.12 & 0.08 & {\bf 0.00} & 0.01 & 0.00 & 0.00 & 0.00 & 0.49 & 0.50 & 0.01 \\
PHI & 0.10 & 0.04 & 0.19 & 0.16 & 0.03 & 0.06 & 0.07 & 0.06 & 0.00 & 0.28 & 0.28 & 0.00\\
GPT & {\bf 0.02} & {\bf 0.01} & 0.06 & {\bf 0.02} & 0.04 & {\bf 0.00} & {\bf 0.00} & {\bf 0.00} & {\bf 0.00} & 0.76 & 0.96 & 0.20\\
\bottomrule

\end{tabular}
\caption{\label{tab:basic_simu} Average $|\bar{\mathcal{E}}- \hat{\mathcal{E}}|$ over simulation 10 iterations. We have three different ratio value run over three different data sizes. For $e$, we use Microsoft Phi-3.5 \cite{abdin2024phi}, OpenAI GPT3.5-turbo \cite{hurst2024gpt}, A21-Jamba 1.5 Mini \cite{lieber2021jurassic}, Cohere-command-r-08-2024 \cite{Ustun2024AyaMA}, Ministral-3B \cite{jiang2023mistral} and the Llama 3.2 70B model \cite{dubey2024llama}. We also present the average $p$, $q$ and $p-q$ observed in the experiments.  }
\end{table*}
\begin{figure}
    \centering
    \includegraphics[width=1\linewidth]{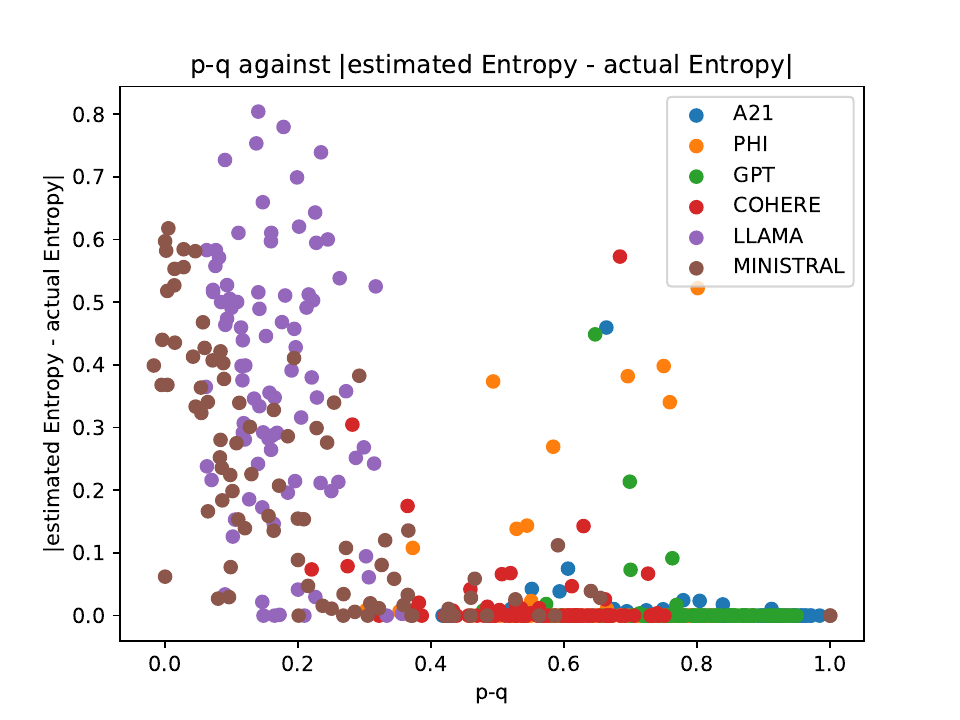}
    \caption{\label{fig:dotdata} A scatter plot of $p-q$ against $|\bar{\mathcal{E}}- \hat{\mathcal{E}}|$ for outcome in experiment 1. The different colors represents different language models used as $e$: A21 in blue, Phi in Orange, GPT in Green, Cohere in Red, Llama is Purple and Ministral in Brown. We notice that there is clear phrase change point where for $p-q <0.4$, we have that $|\bar{\mathcal{E}}- \hat{\mathcal{E}}|$ is very high most of the time, for $p-q >0.4$, $|\bar{\mathcal{E}}- \hat{\mathcal{E}}|$ is small with occasional jumps that the theory predicts.}
    
\end{figure}

We conduct simulation studies for ratio configurations of (0.2, 0.3, 0.5), (0.3, 0.3, 0.4), and (0.5, 0.5), with sample sizes of 30, 50, and 70. The average deviation, \( |\bar{\mathcal{E}}- \hat{\mathcal{E}}| \), over 10 iterations using different models as \( e \) is recorded in Table \ref{tab:basic_simu}.  

In Experiment 1, the algorithm performs strongly when using Cohere, A21, Phi, and GPT, whereas its performance is weaker with Minstral and Llama. In Experiment 2, Llama, Cohere, A21, and GPT demonstrate strong performance, while Minstral and Phi exhibit weaker results. These findings suggest that Llama performs more effectively with longer text, which we primarily attribute to differences in \( p - q \). Specifically, in Experiment 1, \( p - q \) is small for Llama and Minstral but large for Cohere, A21, Phi, and GPT. In Experiment 2, \( p - q \) is small for Phi and Minstral but large for Llama, A21, Cohere, and GPT.  

A comparative analysis of Llama and Minstral in Experiment 1, as well as Phi and Minstral in Experiment 2, indicates that neither \( p \) nor \( q \) alone has a strong effect on \( |\bar{\mathcal{E}}- \hat{\mathcal{E}}| \). Instead, when plotting \( p - q \) against \( |\bar{\mathcal{E}}- \hat{\mathcal{E}}| \) in Figure \ref{fig:dotdata}, we observe a phase change at \( p - q = 0.4 \). Specifically, for \( p - q < 0.4 \), \( |\bar{\mathcal{E}}- \hat{\mathcal{E}}| \) is relatively high, whereas for \( p - q > 0.4 \), \( |\bar{\mathcal{E}}- \hat{\mathcal{E}}| \) tends to be lower. Notably, the experimental setup—defined by the number of data points and cluster ratio—has a considerably weaker effect on \( |\bar{\mathcal{E}}- \hat{\mathcal{E}}| \).

\section{Discussion}
\label{conclusion}
{\color{white} .. }Many natural language processing tasks exhibit a fundamental invariance: sequences of distinct tokens can convey identical meanings. This paper introduces a theoretically grounded metric for quantifying semantic uncertainty, referred to as semantic spectral clustering. This approach reframes the challenge of measuring semantic uncertainty as a prompt-engineering problem, which can be applied to any large language model (LLM), as demonstrated through our simulation analysis. In addition, unsupervised uncertainty can offer a solution to the issue identified in prior research, where supervised uncertainty measures face challenges in handling distributional shifts.
While we define two texts as having equivalent meaning if and only if they mutually imply one another, alternative definitions may be appropriate for specific use cases. For example, legal documents could be clustered based on the adoption of similar legal strategies, with documents grouped together if they demonstrate comparable approaches. In such scenarios, the entropy of the legal documents could also be computed to quantify their informational diversity. We have demonstrated that, provided there exists a function $e$ capable of performing the evaluation with weak accuracy, this estimator remains consistent. Given the reasoning capabilities of large language models (LLMs), we foresee numerous possibilities for extending this method to a wide range of applications.

In addition to the methodology presented, we present a theoretical analysis of the proposed algorithms by proving a theorem concerning the contraction rates of the entropy estimator and its strong consistency. Although the algorithm utilizes generative models, which are typically treated as black-boxes, we simplify the analysis by considering the outputs of these models as random variables. We demonstrate that only a few conditions on the generative are sufficient for our spectral clustering algorithm to achieve strong consistency. Our approach allows for many statistical methodologies to be applied in conjunctions with generative models to analyze text at a level previously not achievable by humans.


\bibliography{custom}

\begin{thebibliography}{40}
\providecommand{\natexlab}[1]{#1}

\bibitem[{Abdin et~al.(2024)Abdin, Aneja, Awadalla, Awadallah, Awan, Bach, Bahree, Bakhtiari, Bao, Behl et~al.}]{abdin2024phi}
Marah Abdin, Jyoti Aneja, Hany Awadalla, Ahmed Awadallah, Ammar~Ahmad Awan, Nguyen Bach, Amit Bahree, Arash Bakhtiari, Jianmin Bao, Harkirat Behl, et~al. 2024.
\newblock Phi-3 technical report: A highly capable language model locally on your phone.
\newblock \emph{arXiv preprint arXiv:2404.14219}.

\bibitem[{Amodei et~al.(2016)Amodei, Olah, Steinhardt, Christiano, Schulman, and Man{\'e}}]{amodei2016concrete}
Dario Amodei, Chris Olah, Jacob Steinhardt, Paul Christiano, John Schulman, and Dan Man{\'e}. 2016.
\newblock Concrete problems in ai safety.
\newblock \emph{arXiv preprint arXiv:1606.06565}.

\bibitem[{Atil et~al.(2024)Atil, Chittams, Fu, Ture, Xu, and Baldwin}]{atil2024llm}
Berk Atil, Alexa Chittams, Liseng Fu, Ferhan Ture, Lixinyu Xu, and Breck Baldwin. 2024.
\newblock Llm stability: A detailed analysis with some surprises.
\newblock \emph{arXiv preprint arXiv:2408.04667}.

\bibitem[{Chen and Lei(2018)}]{chen2018network}
Kehui Chen and Jing Lei. 2018.
\newblock Network cross-validation for determining the number of communities in network data.
\newblock \emph{Journal of the American Statistical Association}, 113(521):241--251.

\bibitem[{Choo et~al.(2020)Choo, Grunau, Portmann, and Rozhon}]{choo2020k}
Davin Choo, Christoph Grunau, Julian Portmann, and V{\'a}clav Rozhon. 2020.
\newblock k-means++: few more steps yield constant approximation.
\newblock In \emph{International Conference on Machine Learning}, pages 1909--1917. PMLR.

\bibitem[{Chowdhery et~al.(2023)Chowdhery, Narang, Devlin, Bosma, Mishra, Roberts, Barham, Chung, Sutton, Gehrmann et~al.}]{chowdhery2023palm}
Aakanksha Chowdhery, Sharan Narang, Jacob Devlin, Maarten Bosma, Gaurav Mishra, Adam Roberts, Paul Barham, Hyung~Won Chung, Charles Sutton, Sebastian Gehrmann, et~al. 2023.
\newblock Palm: Scaling language modeling with pathways.
\newblock \emph{Journal of Machine Learning Research}, 24(240):1--113.

\bibitem[{Copi et~al.(2016)Copi, Cohen, and McMahon}]{copi2016introduction}
Irving~M Copi, Carl Cohen, and Kenneth McMahon. 2016.
\newblock \emph{Introduction to logic}.
\newblock Routledge.

\bibitem[{Desai and Durrett(2020)}]{desai2020calibration}
Shrey Desai and Greg Durrett. 2020.
\newblock Calibration of pre-trained transformers.
\newblock In \emph{Proceedings of the 2020 Conference on Empirical Methods in Natural Language Processing (EMNLP)}, pages 295--302.

\bibitem[{Duan et~al.(2024)Duan, Cheng, Wang, Zavalny, Wang, Xu, Kailkhura, and Xu}]{duan-etal-2024-shifting}
Jinhao Duan, Hao Cheng, Shiqi Wang, Alex Zavalny, Chenan Wang, Renjing Xu, Bhavya Kailkhura, and Kaidi Xu. 2024.
\newblock \href {https://doi.org/10.18653/v1/2024.acl-long.276} {Shifting attention to relevance: Towards the predictive uncertainty quantification of free-form large language models}.
\newblock In \emph{Proceedings of the 62nd Annual Meeting of the Association for Computational Linguistics (Volume 1: Long Papers)}, pages 5050--5063, Bangkok, Thailand. Association for Computational Linguistics.

\bibitem[{Dubey et~al.(2024)Dubey, Jauhri, Pandey, Kadian, Al-Dahle, Letman, Mathur, Schelten, Yang, Fan et~al.}]{dubey2024llama}
Abhimanyu Dubey, Abhinav Jauhri, Abhinav Pandey, Abhishek Kadian, Ahmad Al-Dahle, Aiesha Letman, Akhil Mathur, Alan Schelten, Amy Yang, Angela Fan, et~al. 2024.
\newblock The llama 3 herd of models.
\newblock \emph{arXiv preprint arXiv:2407.21783}.

\bibitem[{Esposito(2024)}]{esposito2024programming}
Francesco Esposito. 2024.
\newblock \emph{Programming Large Language Models with Azure Open AI: Conversational Programming and Prompt Engineering with LLMs}.
\newblock Microsoft Press.

\bibitem[{Glushkova et~al.(2021)Glushkova, Zerva, Rei, and Martins}]{glushkova2021uncertainty}
Taisiya Glushkova, Chrysoula Zerva, Ricardo Rei, and Andr{\'e}~FT Martins. 2021.
\newblock Uncertainty-aware machine translation evaluation.
\newblock \emph{arXiv preprint arXiv:2109.06352}.

\bibitem[{Hendrycks et~al.(2021)Hendrycks, Carlini, Schulman, and Steinhardt}]{hendrycks2021unsolved}
Dan Hendrycks, Nicholas Carlini, John Schulman, and Jacob Steinhardt. 2021.
\newblock Unsolved problems in ml safety.
\newblock \emph{arXiv preprint arXiv:2109.13916}.

\bibitem[{Holland et~al.(1983)Holland, Laskey, and Leinhardt}]{holland1983stochastic}
Paul~W Holland, Kathryn~Blackmond Laskey, and Samuel Leinhardt. 1983.
\newblock Stochastic blockmodels: First steps.
\newblock \emph{Social networks}, 5(2):109--137.

\bibitem[{Hurst et~al.(2024)Hurst, Lerer, Goucher, Perelman, Ramesh, Clark, Ostrow, Welihinda, Hayes, Radford et~al.}]{hurst2024gpt}
Aaron Hurst, Adam Lerer, Adam~P Goucher, Adam Perelman, Aditya Ramesh, Aidan Clark, AJ~Ostrow, Akila Welihinda, Alan Hayes, Alec Radford, et~al. 2024.
\newblock Gpt-4o system card.
\newblock \emph{arXiv preprint arXiv:2410.21276}.

\bibitem[{Jiang et~al.(2023)Jiang, Sablayrolles, Mensch, Bamford, Chaplot, Casas, Bressand, Lengyel, Lample, Saulnier et~al.}]{jiang2023mistral}
Albert~Q Jiang, Alexandre Sablayrolles, Arthur Mensch, Chris Bamford, Devendra~Singh Chaplot, Diego de~las Casas, Florian Bressand, Gianna Lengyel, Guillaume Lample, Lucile Saulnier, et~al. 2023.
\newblock Mistral 7b.
\newblock \emph{arXiv preprint arXiv:2310.06825}.

\bibitem[{Jiang et~al.(2021)Jiang, Araki, Ding, and Neubig}]{jiang2021can}
Zhengbao Jiang, Jun Araki, Haibo Ding, and Graham Neubig. 2021.
\newblock How can we know when language models know? on the calibration of language models for question answering.
\newblock \emph{Transactions of the Association for Computational Linguistics}, 9:962--977.

\bibitem[{Kadavath et~al.(2022)Kadavath, Conerly, Askell, Henighan, Drain, Perez, Schiefer, Hatfield-Dodds, DasSarma, Tran-Johnson et~al.}]{kadavath2022language}
Saurav Kadavath, Tom Conerly, Amanda Askell, Tom Henighan, Dawn Drain, Ethan Perez, Nicholas Schiefer, Zac Hatfield-Dodds, Nova DasSarma, Eli Tran-Johnson, et~al. 2022.
\newblock Language models (mostly) know what they know.
\newblock \emph{arXiv preprint arXiv:2207.05221}.

\bibitem[{Kuhn et~al.(2023)Kuhn, Gal, and Farquhar}]{kuhn2023semantic}
Lorenz Kuhn, Yarin Gal, and Sebastian Farquhar. 2023.
\newblock Semantic uncertainty: Linguistic invariances for uncertainty estimation in natural language generation.
\newblock \emph{arXiv preprint arXiv:2302.09664}.

\bibitem[{Kumar et~al.(2004)Kumar, Sabharwal, and Sen}]{kumar2004simple}
Amit Kumar, Yogish Sabharwal, and Sandeep Sen. 2004.
\newblock A simple linear time (1+/spl epsiv/)-approximation algorithm for k-means clustering in any dimensions.
\newblock In \emph{45th Annual IEEE Symposium on Foundations of Computer Science}, pages 454--462. IEEE.

\bibitem[{Lei(2016)}]{lei2016goodness}
Jing Lei. 2016.
\newblock A goodness-of-fit test for stochastic block models.
\newblock \emph{The Annals of Statistics}, 44(1):401.

\bibitem[{Lei and Rinaldo(2015)}]{lei2015consistency}
Jing Lei and Alessandro Rinaldo. 2015.
\newblock Consistency of spectral clustering in stochastic block models.
\newblock \emph{The Annals of Statistics}, pages 215--237.

\bibitem[{Liebeck(2018)}]{liebeck2018concise}
Martin Liebeck. 2018.
\newblock \emph{A concise introduction to pure mathematics}.
\newblock Chapman and Hall/CRC.

\bibitem[{Lieber et~al.(2021)Lieber, Sharir, Lenz, and Shoham}]{lieber2021jurassic}
Opher Lieber, Or~Sharir, Barak Lenz, and Yoav Shoham. 2021.
\newblock Jurassic-1: Technical details and evaluation.
\newblock \emph{White Paper. AI21 Labs}, 1(9).

\bibitem[{Lin et~al.(2024)Lin, Hilton, and Evans}]{linteaching}
Stephanie Lin, Jacob Hilton, and Owain Evans. 2024.
\newblock Teaching models to express their uncertainty in words.
\newblock \emph{Transactions on Machine Learning Research}.

\bibitem[{Lin et~al.(2023)Lin, Trivedi, and Sun}]{lin2023generating}
Zhen Lin, Shubhendu Trivedi, and Jimeng Sun. 2023.
\newblock Generating with confidence: Uncertainty quantification for black-box large language models.
\newblock \emph{arXiv preprint arXiv:2305.19187}.

\bibitem[{Mavroudis(2024)}]{mavroudis2024langchain}
Vasilios Mavroudis. 2024.
\newblock Langchain.

\bibitem[{Nikitin et~al.(2024)Nikitin, Kossen, Gal, and Marttinen}]{nikitin2024kernel}
Alexander Nikitin, Jannik Kossen, Yarin Gal, and Pekka Marttinen. 2024.
\newblock Kernel language entropy: Fine-grained uncertainty quantification for llms from semantic similarities.
\newblock \emph{arXiv preprint arXiv:2405.20003}.

\bibitem[{Nvidia(2024)}]{nvidia2024}
Nvidia. 2024.
\newblock \href {https://docs.nvidia.com/cuda/floating-point/index.html} {Floating point and ieee 754 compliance for nvidia gpus}.

\bibitem[{Pedregosa et~al.(2011)Pedregosa, Varoquaux, Gramfort, Michel, Thirion, Grisel, Blondel, Prettenhofer, Weiss, Dubourg, Vanderplas, Passos, Cournapeau, Brucher, Perrot, and Duchesnay}]{scikit-learn}
F.~Pedregosa, G.~Varoquaux, A.~Gramfort, V.~Michel, B.~Thirion, O.~Grisel, M.~Blondel, P.~Prettenhofer, R.~Weiss, V.~Dubourg, J.~Vanderplas, A.~Passos, D.~Cournapeau, M.~Brucher, M.~Perrot, and E.~Duchesnay. 2011.
\newblock Scikit-learn: Machine learning in {P}ython.
\newblock \emph{Journal of Machine Learning Research}, 12:2825--2830.

\bibitem[{Shi and Malik(2000)}]{shi2000normalized}
Jianbo Shi and Jitendra Malik. 2000.
\newblock Normalized cuts and image segmentation.
\newblock \emph{IEEE Transactions on pattern analysis and machine intelligence}, 22(8):888--905.

\bibitem[{Song et~al.(2024)Song, Wang, Li, and Lin}]{song2024good}
Yifan Song, Guoyin Wang, Sujian Li, and Bill~Yuchen Lin. 2024.
\newblock The good, the bad, and the greedy: Evaluation of llms should not ignore non-determinism.
\newblock \emph{arXiv preprint arXiv:2407.10457}.

\bibitem[{Su et~al.(2019)Su, Wang, and Zhang}]{su2019strong}
Liangjun Su, Wuyi Wang, and Yichong Zhang. 2019.
\newblock Strong consistency of spectral clustering for stochastic block models.
\newblock \emph{IEEE Transactions on Information Theory}, 66(1):324--338.

\bibitem[{Touvron et~al.(2023)Touvron, Lavril, Izacard, Martinet, Lachaux, Lacroix, Rozi{\`e}re, Goyal, Hambro, Azhar et~al.}]{touvron2023llama}
Hugo Touvron, Thibaut Lavril, Gautier Izacard, Xavier Martinet, Marie-Anne Lachaux, Timoth{\'e}e Lacroix, Baptiste Rozi{\`e}re, Naman Goyal, Eric Hambro, Faisal Azhar, et~al. 2023.
\newblock Llama: Open and efficient foundation language models.
\newblock \emph{arXiv preprint arXiv:2302.13971}.

\bibitem[{Ustun et~al.(2024)Ustun, Aryabumi, Yong, Ko, D'souza, Onilude, Bhandari, Singh, Ooi, Kayid, Vargus, Blunsom, Longpre, Muennighoff, Fadaee, Kreutzer, and Hooker}]{Ustun2024AyaMA}
A.~Ustun, Viraat Aryabumi, Zheng-Xin Yong, Wei-Yin Ko, Daniel D'souza, Gbemileke Onilude, Neel Bhandari, Shivalika Singh, Hui-Lee Ooi, Amr Kayid, Freddie Vargus, Phil Blunsom, Shayne Longpre, Niklas Muennighoff, Marzieh Fadaee, Julia Kreutzer, and Sara Hooker. 2024.
\newblock \href {https://api.semanticscholar.org/CorpusID:267627803} {Aya model: An instruction finetuned open-access multilingual language model}.
\newblock In \emph{Annual Meeting of the Association for Computational Linguistics}.

\bibitem[{Wang et~al.(2023)Wang, Liu, and Awadallah}]{wang2023cost}
Chi Wang, Xueqing Liu, and Ahmed~Hassan Awadallah. 2023.
\newblock Cost-effective hyperparameter optimization for large language model generation inference.
\newblock In \emph{International Conference on Automated Machine Learning}, pages 21--1. PMLR.

\bibitem[{Wang et~al.(2020)Wang, Hsieh, Chang, Chen, Pan, Wei, and Juan}]{wang2020contextual}
Pei-Hsin Wang, Sheng-Iou Hsieh, Shih-Chieh Chang, Yu-Ting Chen, Jia-Yu Pan, Wei Wei, and Da-Chang Juan. 2020.
\newblock Contextual temperature for language modeling.
\newblock \emph{arXiv preprint arXiv:2012.13575}.

\bibitem[{Wang and Bickel(2017)}]{wang2017likelihood}
YX~Rachel Wang and Peter~J Bickel. 2017.
\newblock Likelihood-based model selection for stochastic block models.
\newblock \emph{The Annals of Statistics}, 45(2):500.

\bibitem[{{Wikipedia contributors}(2024)}]{wiki}
{Wikipedia contributors}. 2024.
\newblock \href {https://en.wikipedia.org/wiki/December_3} {December 3}.
\newblock [Online; accessed 03-Dec-2004].

\bibitem[{Wittgenstein(2023)}]{wittgenstein2023tractatus}
Ludwig Wittgenstein. 2023.
\newblock \emph{Tractatus logico-philosophicus}.
\newblock Linkgua.

\end{thebibliography}

\onecolumn
\appendix
\section{Limitation}
{\color{white} .. }We acknowledge that, while this research offers a theoretically consistent measurement of variation, it does not account for situations where two pieces of text may partially agree. For example, two texts may contain points of agreement and points of disagreement. This is particularly common when different authors cite similar sources, but reach contradictory conclusions. Computing semantic similarity in this case is difficult.
\section{Theoretical Result}
\subsection{Proof of proposition  \ref{prop:equ}}
\begin{proof} 
    To prove that the relation $t_i \sim t_j$ if $t_i$ is true if and only if $t_j$ is true is an equivalence relation, we need to meet 3 key criteria, namely symmetry, reflexivity, and Transitivity. 

    First, symmetry 
    $t_i \sim t_j$ implies that $t_j$ is true $\Leftrightarrow$ $t_j$ is true, but this also means $t_j$ is true $\Leftrightarrow$ $t_i$ is true. Then we have $t_j \sim t_i$. 

    Second, reflexivity, 
    $t_i \sim t_j$ implies $t_j$ is true $\Leftrightarrow$ $t_j$ is true. But this means that $t_j$ is true  $\Leftrightarrow$ $t_i$ is true. Then we have $t_i \sim t_j$. 
    
    Third, transitivity,
    If $t_i \sim t_j$ and $t_j \sim t_k$, Then if $t_i$ is true $\Rightarrow$ $t_j$ is true $\Rightarrow$ $t_k$ is true, which means $t_i$ is true $\Rightarrow$ $t_k$ is true. On the other hand, using the same argument, $t_k$ is true $\Rightarrow$ $t_j$ is true $\Rightarrow$ $t_i$ is true. This means the $t_k$ is true $\Rightarrow$ $t_i$ is true. Therefore $t_i \sim t_k$. 

    The three points is sufficient to demonstrate that $\sim$ is a equivalence relation. 
\end{proof}
\subsection{Proof of Theorem \ref{the:strongConsistensy}}
\label{appendix:sec:the:strongconsistency}
To prove Theorem \ref{the:strongConsistensy}, we adopt notations from \cite{su2019strong}.
Consider the adjacency matrix $E$ which is determined by a Language model. 

Let $d_i = \sum_{j=1}^n E_{ij}$ denote the degree of node $i$,  $D = \text{diag}(d_1,\cdots, d_n)$, and $L = D^{-1/2}ED^{-1/2}$ be the graph Laplacian. We also define $n_k$ be the number of text in each cluster. We denote a block probability matrix $B = B_{k_1k_2}$ where $k_1,k_2 \in\{1,\cdots K\}$ be the clusters index.  i.e. 
$$ B_{k_1 k_2} = \begin{cases}
    p \quad \text{if $k_1 = k_2$}\\
    1-q \quad \text{otherwise.}
\end{cases}$$

Let $\mathbb{E}(E) = P$ i.e. the probability of edge between $i$ and $j$ is given by $P_{ij} = B_{k_1k_2}$ if text $i$ is in $C_{k_1}$ and $j$ is in $C_{k_2}$.
Denote $Z = \{Z_{ik}\}$ be a $n\times K$  binary matrix providing the cluster membership of text $t$, i.e., $Z_{ik} = 1$ if text $i$ is in $C_k$ and $Z_{ik} = 0$ otherwise. The population version of the Laplacian is given by $\mathcal{L} = \mathcal{D}^{-1/2}P\mathcal{D}^{-1/2}$ where  $\mathcal{D} = \text{diag}(d_1 \cdots d_n)$ where $d_i =\sum_{j=1}^{n}P_{ij} = p + (n-1)(q)$.

Let $\pi_{kn} = n_k/n, W_k = \sum_{l=1}^KB_{kl}\pi_{ln}$, $\mathcal{D}_B = \text{diag}(W_1,\cdots W_K)$, and $B_0=\mathcal{D}_B^{-1/2}B\mathcal{D}_B^{-1/2}  $
\begin{assumption}[Assumption 1 in \cite{su2019strong}]
\label{assumption:eigenvalues}
$P$ is rank $k$ and spectral decomposition $\Pi_{n}^{1/2}P\Pi_{n}^{1/2}$ is $S_n \Omega_n S_n^T$ in which $S_n$ is a $K \times K$ matrix such that $S_n^T S_n = I_{K\times K}$  and $\Omega_n = \text{diag}(\omega_1 \cdots \omega_{K_n})$ such that $|\omega_1|\geq |\omega_2|\geq\cdots \geq|\omega_{K_n}|$
\end{assumption}
Assumption \ref{assumption:eigenvalues} implies that the spectral decomposition $$\mathcal{L} = U_n \Sigma_n U_n^T = U_{1n}\Sigma_{1n}U_{1n}^T$$

where \(\Sigma_{n}=\operatorname{diag}\left(\sigma_{1 n}, \ldots, \sigma_{K n}, 0, \ldots, 0\right)\) is a \(n \times n\) matrix that contains the eigenvalues of \(\mathcal{L}\) such that \(\left|\sigma_{1 n}\right| \geq\left|\sigma_{2 n}\right| \geq \cdots \geq\left|\sigma_{K n}\right|>0, \Sigma_{1 n}=\operatorname{diag}\left(\sigma_{1 n}, \ldots, \sigma_{K n}\right)\), the columns of \(U_{n}\) contain the 
 eigenvectors of \(\mathcal{L}\) associated with the eigenvalues in \(\Sigma_{n}, U_{n}=\left(U_{1 n}, U_{2 n}\right)\), and \(U_{n}^{T} U_{n}=I_{n}\) \cite{su2019strong}.
\begin{assumption}[Assumption 2 in \cite{su2019strong}]
\label{assumption:limits_nk}
    There exists constant $C_1 >0$ and $c_2>0$ such that
    $$C_1 \geq \lim\sup_n\sup_k n_k K/n \geq \lim \inf_n \inf_k n_k K/n \geq c_2  $$
\end{assumption}

\begin{assumption}[Assumption 3 in \cite{su2019strong}]
\label{assumption:bound_eigenvalues}
    Let $\mu_n = \min_i d_i$ and $\rho_n = \max(\sup_{k_1k_2}[B_0]_{k_1k_2},1)$. Then $n$ sufficiently large, 
    $$ 
\frac{K \rho_{n} \log ^{1 / 2}(n)}{\mu_{n}^{1 / 2} \sigma_{K n}^{2}}\left(1+\rho_{n}+\left(\frac{1}{K}+\frac{\log (5)}{\log (n)}\right)^{1 / 2} \rho_{n}^{1 / 2}\right) \leq 10^{-8} C_{1}^{-1} c_{2}^{1 / 2} .
$$
    
\end{assumption}
Let 
$$ 
\hat{O}_{n}=\bar{U} \bar{V}^{T}
$$
where \(\bar{U} \bar{\Sigma} \bar{V}^{T}\) is the singular value decomposition of \(\hat{U}_{1 n}^{T} U_{1 n}\). we also denote \(\hat{u}_{1 i}^{T}\) and \(u_{1 i}^{T}\) as the \(i\)-th rows of \(\hat{U}_{1 n}\) and \(U_{1 n}\), respectively.

Now we present the notation of the K-means algorithm. With a little abuse of notation, let \(\hat{\beta}_{\text {in }} \in \mathbb{R}^{K}\) be a generic estimator of \(\beta_{g_{i}^{0} n} \in \mathbb{R}^{K}\) for \(i=1, \ldots, n\). To recover the community membership structure (i.e., to estimate \(g_{i}^{0}\) ), it is natural to apply the  K-means clustering algorithm to \(\left\{\widehat{\beta}_{\text {in }}\right\}\). Specifically, let \(\mathcal{A}=\left\{\alpha_{1}, \ldots, \alpha_{K}\right\}\) be a set of \(K\) arbitrary  \(K \times 1\) vectors: \(\alpha_{1}, \ldots, \alpha_{K}\). Define
\[
\widehat{Q}_{n}(\mathcal{A})=\frac{1}{n} \sum_{i=1}^{n} \min _{1 \leq l \leq K}\left\|\hat{\beta}_{i n}-\alpha_{l}\right\|^{2}
\]

and \(\widehat{\mathcal{A}}_{n}=\left\{\widehat{\alpha}_{1}, \ldots, \widehat{\alpha}_{K}\right\}\), where \(\widehat{\mathcal{A}}_{n}=\arg \min _{\mathcal{A}} \widehat{Q}_{n}(\mathcal{A})\). Then we compute the estimated cluster  identity as
\[
\hat{g}_{i}=\underset{1 \leq l \leq K}{\arg \min }\left\|\hat{\beta}_{\text {in }}-\widehat{\alpha}_{l}\right\|,
\]

where if there are multiple \(l\) 's that achieve the minimum, \(\hat{g}_{i}\) takes value of the smallest one. We then state the key assumption that relates to K-means clustering algorithm. 

\begin{assumption}[Assumption 7 in \cite{su2019strong}]
\label{assumption:K-means}
     Suppose for \(n\) sufficiently large,
     \[
15 C^{*} \frac{K \rho_{n} \log ^{1 / 2}(n)}{\mu_{n}^{1 / 2} \sigma_{K n}^{2}}\left(1+\rho_{n}+\left(\frac{1}{K}+\frac{\log (5)}{\log (n)}\right)^{1 / 2} \rho_{n}^{1 / 2}\right) \leq c_{2} C_{1}^{-1 / 2} \sqrt{2}
\]
Where \(C^{*} = 3528C_1 c_2^{-1/2} \)
\end{assumption}

\begin{theorem}(Collorary 2.2)
\label{theorem:no_error}
    Corollary 2.2. Suppose that Assumptions \ref{assumption:eigenvalues},  \ref{assumption:limits_nk}, \ref{assumption:bound_eigenvalues}, and \ref{assumption:K-means} hold and the \(K\)-means algorithm is applied  to \(\hat{\beta}_{i n}=(n / K)^{1 / 2} \hat{u}_{1 i}\) and \(\beta_{g_{i}^{0} n}=(n / K)^{1 / 2} \hat{O}_{n} u_{1 i}\) Then, 
    \[
\sup _{1 \leq i \leq n} \mathbf{1}\left\{\tilde{g}_{i} \neq g_{i}^{0}\right\}=0 \quad \text { a.s. }
\]
\end{theorem}

We now have define the error of mis-classification. 

\begin{definition}
Denote $M_\text{error} = \sum_{j} \sum_{i}\mathbb{I}(g_{ij}  \neq g^{\text{True}}_{ij})$ as the mis-classification error.
\end{definition}

\begin{lemma}
\label{lemma:error_connections}
    If $\sup_{i,j} \mathbb{I}(g_{ij}  \neq g^{\text{True}}_{ij}) = 0 \quad \text{a.s.}$, then $M_\text{error} = 0 \quad \text{a.s.}$
\end{lemma}
\begin{proof}
Notice $\mathbb{I}(g_{ij}  \neq g^{\text{True}}_{ij})$ can only takes up value $1$ or $0$. Therefore $\sum_{j} \sum_{i}\mathbb{I}(g_{ij}\neq g^{\text{True}}_{ij}) \neq 0 \Leftrightarrow \exists i, j  \text{ s.t }\mathbb{I}(g_{ij}\neq g^{\text{True}}_{ij}) \neq 0 \Leftrightarrow  \sup_{i,j}\mathbb{I}(g_{ij}\neq g^{\text{True}}_{ij}) \neq 0$  
    \begin{align*}
        \mathbb{P}(M_\text{error} \neq 0 \text{ i.o. }) &= \mathbb{P}\left( \sum_{j} \sum_{i}\mathbb{I}(g_{ij}\neq g^{\text{True}}_{ij}) \neq 0 \text{ i.o.}\right)\\
        &= \mathbb{P}\left( \exists i, j  \text{ s.t }\mathbb{I}(g_{ij}\neq g^{\text{True}}_{ij}) \neq 0 \text{ i.o. } \right)\\
        &= \mathbb{P}\left( \sup_{i,j}\mathbb{I}(g_{ij}\neq g^{\text{True}}_{ij}) \neq 0 \text{ i.o }\right)\\
        &= 0 \quad \text{ since $\sup_{i,j} \mathbb{I}(g_{ij}  \neq g^{\text{True}}_{ij}) = 0$ \text{ a.s.}}
    \end{align*}
    Here we use the classical notation i.o. as happens infinitely often. 
\end{proof}
\begin{lemma}
    \label{lemma:misclassification}
    $\sum_j \left|\sum_{i=1}^n g_{ij}-n_j\right| \leq M_\text{error}$
\end{lemma}
\begin{proof}

\begin{align*}
\sum_j \left|\sum_{i=1}^n g_{ij}-n_j\right|
&= \sum_j \Biggl| \sum_i \mathbb{I}(g_{ij} = 1, g^{\text{True}}_{ij} = 0 ) + \mathbb{I}(g_{ij} = 1, g^{\text{True}}_{ij} = 1 ) + \mathbb{I}(g_{ij} = 0, g^{\text{True}}_{ij} = 1 ) \\
&- \mathbb{I}(g_{ij} = 0, g^{\text{True}}_{ij} = 1 )- n_j \Biggr|\\
& = \sum_j \Biggl| \sum_i \mathbb{I}(g_{ij} = 1, g^{\text{True}}_{ij} = 0 ) - \mathbb{I}(g_{ij} = 0, g^{\text{True}}_{ij} = 1 ) \\
& + \sum_i \mathbb{I}(g_{ij} = 0, g^{\text{True}}_{ij} = 1 )+ \mathbb{I}(g_{ij} = 0, g^{\text{True}}_{ij} = 1 ) - n_j \Biggr|\\
& = \sum_j \Biggl| \sum_i \mathbb{I}(g_{ij} = 1, g^{\text{True}}_{ij} = 0 ) - \mathbb{I}(g_{ij} = 0, g^{\text{True}}_{ij} = 1 ) + n_j - n_j \Biggr|\\
&= \sum_j \Biggl| \sum_i \mathbb{I}(g_{ij} = 1, g^{\text{True}}_{ij} = 0 ) -\mathbb{I}(g_{ij} = 0, g^{\text{True}}_{ij} = 1 ) \Biggr|\\
&\leq \sum_j\sum_i \mathbb{I}(g_{ij} = 1, g^{\text{True}}_{ij} = 0 ) + \mathbb{I}(g_{ij} = 0, g^{\text{True}}_{ij} = 1 )\\
&=   \sum_{j} \sum_{i}\mathbb{I}(g_{ij}  \neq g^{\text{True}}_{ij}) \\
&= M_\text{error}
\end{align*}
\end{proof}
\newpage
\subsubsection{Proof of lemma \ref{lemma:error}}
\label{Appendix:proofoflemma:error}
Now we prove lemma \ref{lemma:error}.
\begin{proof}
Recall that
\begin{itemize}
    \item $\hat{p}(C_j) = \frac{1}{n}\sum_{i=1}^n g_{ij}$ and $\hat{\mathcal{E}}(\mathcal{T}) = - \sum_{j=1}^K\hat{p}( C_j) \log(\hat{p}( C_j))$
    \item $\bar{p}(C_j) = \frac{n_j}{n}$ and $\bar{\mathcal{E}}(\mathcal{T}) = - \sum_{j=1}^K\bar{p}( C_j) \log(\bar{p}( C_j))$
\end{itemize}
\begin{align*}
    |\hat{\mathcal{E}}(\mathcal{T}) - \bar{\mathcal{E}}(\mathcal{T})| &= \left| \sum_{j=1}^k\hat{p}( C_j) \log(\hat{p}( C_j)) -  \bar{p}( C_j) \log(\bar{p}( C_j)) \right|\\
    &=  \left|\sum_{j=1}^K\hat{p}( C_j)\log(\hat{p}( C_j)) -  \hat{p}( C_j)\log(\bar{p}( C_j)) + \hat{p}( C_j)\log(\bar{p}( C_j)) -  \bar{p}( C_j) \log(\bar{p}( C_j)) \right|\\
    &= \left|\sum_{j=1}^K \hat{p}( C_j)\log\left(\frac{\hat{p}( C_j)}{\bar{p}( C_j)}\right) - \left(\hat{p}( C_j) -\bar{p}( C_j)\right)\log(\bar{p}( C_j)) \right|\\
    &=  \left|\sum_{j=1}^K \hat{p}( C_j)\log\left(\frac{\hat{p}( C_j)}{p( C_j)}\right) - \left(\hat{p}( C_j) -\bar{p}( C_j)\right)\log(\bar{p}( C_j)) \right|\\
    &\leq \left|\sum_{j=1}^K \hat{p}( C_j)\log\left(\frac{\hat{p}( C_j)}{\bar{p}( C_j)}\right)\right| + \left|\sum_{j=1}^K \left(\hat{p}( C_j) -\bar{p}( C_j)\right)\log(\bar{p}( C_j)) \right|\\
    &\leq  \left|\sum_{j=1}^K \left(\frac{\hat{p}( C_j)-\bar{p}( C_j)}{\bar{p}( C_j)}\right)\right| + \left|\sum_{j=1}^K \left(\hat{p}( C_j) -\bar{p}( C_j)\right)\log(\bar{p}( C_j)) \right|\\
    &= \left|\sum_{j=1}^K \left(\frac{\frac{1}{n}\sum_{i=1}^n g_{ij}-\bar{p}( C_j)}{\bar{p}( C_j)}\right)\right| + \left|\sum_{j=1}^K \left(\frac{1}{n}\sum_{i=1}^n g_{ij} -\bar{p}( C_j)\right)\log(\bar{p}( C_j)) \right|\\
    &= \left|\sum_{j=1}^K \left(\frac{\frac{1}{n}\left(\sum_{i=1}^n g_{ij}-n_j\right)}{\bar{p}( C_j)}\right)\right| + \left|\sum_{j=1}^K \left(\frac{1}{n}\sum_{i=1}^n g_{ij} -\bar{p}( C_j)\right)\log(\bar{p}( C_j)) \right|\\
    &\leq \sum_{j=1}^K \left|\frac{\frac{1}{n}\left(\sum_{i=1}^n g_{ij}-n_j\right)}{\bar{p}( C_j)}\right| + \sum_{j=1}^K \left|\frac{1}{n}\sum_{i=1}^n (g_{ij} - n_j)\right|\left|\log(\bar{p}( C_j)) \right|\\
    &\leq \left|\frac{\frac{2K}{n}(M_\text{error})}{c_2}\right| + \log\left(\frac{2K}{c_2}\right) \left|\frac{1}{n} (M_\text{error})\right|\\
    &= h\left(\frac{2K}{c_2}\right) \left|\frac{1}{n} (M_\text{error})\right| 
\end{align*}

where $h(x) = \left(x+\log\left(x\right)\right)$. 
\end{proof}

We prove Theorem \ref{the:strongConsistensy}. To do so, we first restate Theorem \ref{the:strongConsistensy} with all the conditions required to get to the outcome.
\begin{theorem}[Theorem \ref{the:strongConsistensy} with all conditions stated]
    Assume that Assumptions \ref{assumption:eigenvalues},  \ref{assumption:limits_nk}, \ref{assumption:bound_eigenvalues}, and \ref{assumption:K-means} hold and the \(K\)-means algorithm is applied  to \(\hat{\beta}_{i n}=(n / K)^{1 / 2} \hat{u}_{1 i}\) and \(\beta_{g_{i}^{0} n}=(n / K)^{1 / 2} \hat{O}_{n} u_{1 i}\) Then 
    \[ |\bar{\mathcal{E}}(\mathcal{T}) - \hat{\mathcal{E}}(\mathcal{T}) | \rightarrow 0 \text{ almost surely }\]
\end{theorem}

\begin{proof}

Using Theorem \ref{theorem:no_error}, we know that under Assumptions \ref{assumption:eigenvalues},  \ref{assumption:limits_nk}, \ref{assumption:bound_eigenvalues}, and \ref{assumption:K-means}, we have that 
 \[
\sup _{1 \leq i \leq n} \mathbf{1}\left\{\tilde{g}_{i} \neq g_{i}^{0}\right\}=0 \quad \text { a.s. }
\]
Using Lemma \ref{lemma:error_connections}, we know that 

$$M_\text{error} = 0 \quad \text{a.s.}$$
Using results from Lemma \ref{lemma:error}, we know that $M_\text{error} \rightarrow 0 \quad a.s. \Rightarrow \hat{\mathcal{E}}(\mathcal{T}) \rightarrow \bar{\mathcal{E}}(\mathcal{T}) \quad a.s. $. 
\end{proof}
Now we try to prove Theorem \ref{the:finite_sample}. To do so, we state corollary 3.2 in \cite{lei2015consistency}.
\subsection{Proof of Theorem \ref{the:finite_sample}}
\label{appendix:proofofthe:finite_sample}
\begin{theorem}[Corollary 3.2 in \cite{lei2015consistency}]
\label{the:finite_sample_core}
 Let $E$ be an adjacency matrix from the $\operatorname{SBM}(Z, B)$, where $B=\alpha_{n} B_{0}$ for some $\alpha_{n} \geq \log n / n$ and with $B_{0}$ having minimum absolute eigenvalue $\geq \lambda>0$ and $\max _{k \ell} B_{0}(k, \ell)=1$. Let $g_{ij}$ be the output of spectral clustering using $(1+\varepsilon)$-approximate $k$-means. Then  there exists an absolute constant $c$ such that if 

\begin{equation*}
(2+\varepsilon) \frac{K n}{n_{\min }^{2} \lambda^{2} \alpha_{n}}<c
\end{equation*}
then with probability at least $1-n^{-1}$,
$$
\frac{1}{n}M_{\text{error}}\leq c^{-1}(2+\varepsilon) \frac{K n_{\max }}{n_{\min }^{2} \lambda^{2} \alpha_{n}}
$$
\end{theorem}

\begin{proof}
We now prove Theorem  \ref{the:finite_sample}.

    Under the model we have, we know that minimum eigenvalue of $B$ is $\lambda$. Use theorem \ref{the:finite_sample_core} to replace $h\left(\frac{2K}{c_2}\right) \left|\frac{1}{n} (M_\text{error})\right|$ with $ h\left(\frac{2K}{c_2}\right) c^{-1}(2+\varepsilon) \frac{K n_{\max }}{n_{\min }^{2} \lambda^{2} \alpha_{n}} $ in lemma \ref{lemma:error}.

We now have to show the existence of $c$ in Theorem \ref{the:finite_sample_core}.

\begin{align*}
    &\quad 2Kn_{\min}/n \geq c_2 \\
    &\Rightarrow 1/n_{\min}^2 \leq 4K^2/n^2c_2^2\\
    &\Rightarrow (2+\epsilon)\frac{Kn}{n_{\min}^2\lambda^2 \alpha_n} \leq (2+\epsilon)\frac{4K^3 }{n\lambda^2 \alpha_nc_2^2}\leq (2+\epsilon)\frac{4K^3}{\lambda^2c_2^2}\\
    &\text{Let $c = (2+\epsilon)\frac{4K^3}{\lambda^2}c_2^2$}
\end{align*}
substitute $c$ to $ h\left(\frac{2K}{c_2}\right) c^{-1}(2+\varepsilon) \frac{K n_{\max }}{n_{\min }^{2} \lambda^{2} \alpha_{n}} $, we have that 
\begin{equation*}
|\bar{\mathcal{E}}(\mathcal{T}) - \hat{\mathcal{E}}(\mathcal{T}) |  \leq h\left(\frac{2K}{c_2}\right) \frac{n_{\max }}{4c_2^2n_{\min }^{2} \alpha_{n}K^2}
\end{equation*}

\end{proof}
\subsubsection{Proof of Corollary \ref{corollary:rate}}
\label{appendix:proofofcorollary:rate}
\begin{proof}
Now we prove Corollary \ref{corollary:rate}. Note that $n \geq n_{\max} \geq n_{\min} \geq nc_2/2K$.

\begin{equation*}
|\bar{\mathcal{E}}(\mathcal{T}) - \hat{\mathcal{E}}(\mathcal{T}) |  \leq h\left(\frac{2K}{c_2}\right) \frac{n_{\max }}{4c_2^2n_{\min }^{2} \alpha_{n}K^2}\leq h\left(\frac{2K}{c_2}\right) \frac{1}{c_2^4 \alpha n}
\end{equation*}
\end{proof}

\newpage
\begin{lemma}
\label{lemma:final1}
    $$|\mathcal{E} - \hat{\mathcal{E}}| \leq \sum_{j=1}^K \left( \left| \frac{p(C_j) - \bar{p}(C_j)}{p(C_j)}\right| + \log\left(\frac{1}{p(C_j)}\right)\left| p(C_j) - \bar{p}(C_j)\right|\right) + h\left(\frac{2K}{c_2}\right) \left|\frac{1}{n} (M_\text{error})\right| $$
\end{lemma}
\begin{proof}
    First, we have that 
    $$|\mathcal{E} - \hat{\mathcal{E}}| \leq |\mathcal{E} -\bar{\mathcal{E}} +\bar{\mathcal{E}}-  \hat{\mathcal{E}}| \leq |\mathcal{E} -\bar{\mathcal{E}}| + | \bar{\mathcal{E}}-  \hat{\mathcal{E}}| \leq |\mathcal{E} -\bar{\mathcal{E}}| + h\left(\frac{2K}{c_2}\right) \left|\frac{1}{n} (M_\text{error})\right| $$
    Next, 
    \begin{align*}
        |\mathcal{E} -\bar{\mathcal{E}}| &\leq  \left| \sum_{j=1}^kp( C_j) \log(p( C_j)) -  \bar{p}( C_j) \log(\bar{p}( C_j)) \right|\\  
        &\leq \left| \sum_{j=1}^kp( C_j) \log(p( C_j)) -  \bar{p}( C_j) \log(p( C_j)) + \bar{p}( C_j) \log(p( C_j)) - \bar{p}( C_j) \log(\bar{p}( C_j)) \right|\\
        &\leq \sum_{j=1}^k \left|p( C_j) \log(p( C_j)) -  \bar{p}( C_j) \log(p( C_j)) \right| + \left|\bar{p}( C_j) \log(p( C_j)) - \bar{p}( C_j) \log(\bar{p}( C_j)) \right| \\
        &\leq \sum_{j=1}^k \left|p( C_j)  -  \bar{p}( C_j)\right| \log\left(\frac{1}{p( C_j)}\right)  + \left| \frac{p( C_j) -  \bar{p}( C_j)}{p( C_j)} \right|
    \end{align*}
\end{proof}

\begin{lemma}
\label{lemma:Final2}
With probability at least $1-\frac{1}{n}$,
$$\sum_{j=1}^k \left|p( C_j)  -  \bar{p}( C_j)\right| \leq K\sqrt{\frac{1}{2n}\log(2Kn)} $$
\end{lemma}
\begin{proof}

       $$ \left|p( C_j)  -  \bar{p}( C_j)\right| = \frac{1}{n}\left|np( C_j)  -  n_j\right|$$
Now use Hoeffding bound, we notice that for any $j$
$$\mathbb{P}(|n_j - np(C_j)| \geq \delta) \leq 2\exp\left(-\frac{2\delta^2}{n}\right) $$
Using union bound 
$$\mathbb{P}(\exists j \text{ such that }|n_j - np(C_j)| \geq \delta) \leq \sum_{j=1}^K\mathbb{P}(|n_j - np(C_j)| \geq \delta) \leq 2K\exp\left(-\frac{2\delta^2}{n}\right) $$

$\exists j \text{ such that }|n_j - np(C_j)| \geq \delta \Leftarrow\max |n_j - np(C_j)| \geq \delta \Leftarrow \sum_{j=1}^K |n_j - np(C_j)| \geq K\delta.$ 

Now, let $ 2K\exp\left(-\frac{2\delta^2}{n}\right) = \frac{1}{n}$, we have that $\delta = \sqrt{\frac{n}{2}\log(2Kn)}$

This gives us that with probability at least $1-\frac{1}{n}$,

$$ \sum_{j=1}^k \left|p( C_j)  -  \bar{p}( C_j)\right| \leq K\sqrt{\frac{1}{2n}\log(2Kn)} $$ 
\end{proof}
\begin{lemma}
\label{Lemma:Final3}
With probability at least $1-\frac{1}{n}$
$$n_{\min} \geq \frac{nc_2}{2K}$$
where $c_2 = 2K\left(1-\sqrt{\frac{2\log(nK)}{np_{\min}}}\right)p_{\min}$ and $p_{\min} = \min \{p(C_1) \dots p(C_K) \}$
\end{lemma}
\begin{proof}
    Using the Chernoff inequality, we have $$\mathbb{P}\left(n_j \leq (1-\delta)np(C_j)\right) \leq \exp\left(\frac{-np(C_j)}{2}\right)$$
Using the union bound
$$\mathbb{P}(n_{\min} \leq nc_2/2K) \leq \mathbb{P}\left(\exists j \text{ such that }n_j \leq (1-\delta)np(C_j)\right) \leq K\exp\left(\frac{-np_{\min}}{2}\right) $$
Let $K\exp\left(\frac{-np_{\min}}{2}\right) = \frac{1}{n}$, we get $\delta = \sqrt{\frac{2\log(nK)}{np_{\min}}}.$
Finally, we have $c_2 = 2K\left(1-\sqrt{\frac{2\log(nK)}{np_{\min}}}\right)p_{\min}$ 

\end{proof}
\newpage
\section{Simulations}
\subsection{Hobby Examples}
We can consider a list of things that a hypothetical individual "John" likes to do in his free time: 
\begin{itemize}
    \item running / jogging 
    \item Drone flying / pilot Aerial drones
    \item jazzercise / aerobics
    \item making pottery / making ceramics
    \item water gardening / aquatic gardening
    \item caving / spelunking / potholing
    \item cycling / bicycling / biking
    \item reading
    \item writing journals / journal writings/ journaling
    \item sculling / rowing
\end{itemize}

\subsection{Historical Examples}
On the day December 3,
\begin{itemize}
    \item "Pope John X crowns Berengar I of Italy as Holy Roman Emperor"/ "Berengar I of Italy was crowned Emperor by Pope John X."
    \item "USS Alfred becomes the first vessel to fly the Grand Union Flag; the flag is hoisted by John Paul Jones."/
         "John Paul Jones hoisted the the Grand Union Flag on USS Alfred, the first vessel to fly Grand Union Flag."
    \item "French General Jean Victor Marie Moreau decisively defeats the Archduke John of Austria near Munich in the Battle of Hohenlinden. "/"Archduke John of Austria was defeated by French General Jean Victor Marie Moreau near Munich in the Battle of Hohenlinden."
    \item "Illinois becomes the 21st U.S. state."/ "Illinois joined U.S. as its 21st state."
    \item "The Zollverein (German Customs Union) begins the first regular census in Germany"/"The first regular census was conducted by The Zollverein (German Customs Union) in Germany."
    \item "The Duquesne Country and Athletic Club defeats an all-star collection of early football players 16–0."
    \item "Following more than a month of Turkish–Armenian War, the Turkish-dictated Treaty of Alexandropol is concluded."/
          "The Turkish-dictated Treaty of Alexandropol concluded after a month of Turkish–Armenian War."
    \item "President Herbert Hoover delivers his first State of the Union message to Congress. It is presented in the form of a written message rather than a speech."/"President Herbert Hoover presents a written message as his first State of the Union message to Congress rather than a speech."
    \item "The current flag of Singapore is adopted, six months after Singapore became self-governing within the British Empire."/"Singapore adopts it current flag six months after it self-govern within the British Empire."
    \item "Ayatollah Ruhollah Khomeini becomes the first Supreme Leader of Iran.", "Iran has its first Supreme Leader: Ayatollah Ruhollah Khomeini."
\end{itemize}

\newpage
\section{Prompt}
\label{appendix_sec:prompt_engineering}
This is the prompt we inserted for "Phi-3-mini-4k-instruct", "AI21-Jamba-1.5-Mini", "Cohere-command-r-08-2024".

\begin{verbatim}
'''
    You are a expert in logical deduction and you are given 2 piece of texts: TEXT A and TEXT B. 
    You are to identify if TEXT A implies TEXT B and TEXT B implies TEXT A at the same time. 
    
    TEXT A: 
    {text_A}
    
    TEXT B:
    {text_B}
    
    ## OUTPUT
    You are to return TRUE if TEXT A implies TEXT B and TEXT B implies TEXT A at the same time. 
    otherwise, you are to return FALSE 
'''
\end{verbatim}

This is the prompt we inserted for "Ministral-3B","Llama-3.3-70B-Instruct", "gpt-35-turbo"

\begin{verbatim}
''' 
    You are a expert in logical deduction and you are given 2 piece of texts: TEXT A and TEXT B. 
    You are to identify if TEXT A implies TEXT B and TEXT B implies TEXT A at the same time. 
    
    TEXT A: 
    {text_A}
    
    TEXT B:
    {text_B}
    
    ## OUTPUT
    You are to return TRUE if TEXT A implies TEXT B and TEXT B implies TEXT A at the same time. 
    otherwise, you are to return FALSE 
    
    ##FORMAT:
    START with either TRUE or FALSE, then detail your reasoning
'''
\end{verbatim}
\section{Notes on Implementation in python}
The default spectral clustering algorithm in python already implemented in \texttt{k++} algorithm \cite{choo2020k}. 

\end{document}